\documentclass[conference]{IEEEtran}
\IEEEoverridecommandlockouts

\usepackage[utf8]{inputenc}
\usepackage[T1]{fontenc}

\usepackage{cite}
\usepackage{amsmath,amssymb,amsfonts}
\usepackage{amsthm}
\usepackage{algorithm}
\usepackage[noend]{algpseudocode}
\usepackage{graphicx}
\usepackage{textcomp}
\usepackage{xcolor}
\usepackage{booktabs}
\usepackage{enumitem}
\usepackage{array}
\usepackage{balance}
\usepackage{hyperref}
\hypersetup{hidelinks}

\graphicspath{{./}}

\newtheorem{theorem}{Theorem}

\newtheorem{lemma}{Lemma}

\newcommand{\eps}{\varepsilon}
\newcommand{\zetaPQ}{\zeta}
\newcommand{\CLIP}{\textsc{CLIP}}
\newcommand{\CLAP}{\textsc{CLAP}}

\def\BibTeX{{\rm B\kern-.05em{\sc i\kern-.025em b}\kern-.08em
T\kern-.1667em\lower.7ex\hbox{E}\kern-.125emX}}

\begin{document}

\title{LUMA-RAG: Lifelong Multimodal Agents with Provably Stable Streaming Alignment%
}

\author{%
\IEEEauthorblockN{Rohan Wandre}
\IEEEauthorblockA{\textit{Dept. of Computer Engineering} \\
\textit{SIES Graduate School of Technology}\\
Navi Mumbai, India \\
\href{mailto:rohanwandre24@gmail.com}{rohanwandre24@gmail.com}}
\and
\IEEEauthorblockN{Vivek Dhalkari}
\IEEEauthorblockA{\textit{Dept. of Computer Engineering} \\
\textit{SIES Graduate School of Technology}\\
Navi Mumbai, India \\
\href{mailto:vivekdhalkari@gmail.com}{vivekdhalkari@gmail.com}}
\and
\IEEEauthorblockN{Yash Gajewar}
\IEEEauthorblockA{\textit{Dept. of Computer Engineering} \\
\textit{Bharatiya Vidya Bhavan's Sardar Patel}\\
\textit{Institute of Technology}\\
Mumbai, India \\
\href{mailto:yashgajewar06@gmail.com}{yashgajewar06@gmail.com}}
\and
\IEEEauthorblockN{Dr. Namrata Patel}
\IEEEauthorblockA{\textit{Dept. of Computer Engineering} \\
\textit{SIES Graduate School of Technology}\\
Navi Mumbai, India \\
\href{mailto:namratap@sies.edu.in}{namratap@sies.edu.in}}
}
\maketitle

\begin{abstract}
Retrieval-Augmented Generation (RAG) has emerged as the dominant paradigm for grounding large language model outputs in verifiable evidence. However, as modern AI agents transition from static knowledge bases to continuous multimodal streams encompassing text, images, video, and audio, two critical challenges arise: maintaining index freshness without prohibitive re-indexing costs, and preserving cross-modal semantic consistency across heterogeneous embedding spaces. We present LUMA-RAG, a lifelong multimodal agent architecture featuring three key innovations: (i) a streaming, multi-tier memory system that dynamically spills embeddings from a hot HNSW tier to a compressed IVFPQ tier under strict memory budgets; (ii) a streaming \CLAP{}$\rightarrow$\CLIP{} alignment bridge that maintains cross-modal consistency through incremental orthogonal Procrustes updates; and (iii) stability-aware retrieval telemetry providing Safe@k guarantees by jointly bounding alignment drift and quantization error. Experiments demonstrate robust text-to-image retrieval (Recall@10 = 0.94), graceful performance degradation under product quantization offloading, and provably stable audio-to-image rankings (Safe@1 = 1.0), establishing LUMA-RAG as a practical framework for production multimodal RAG systems.
\end{abstract}


\section{Introduction}
\label{sec:intro}
The evolution of Large Language Models (LLMs) from static repositories to continuously learning agents requires systems that ingest heterogeneous multimodal streams while sustaining latency, accuracy, and semantic coherence. RAG~\cite{b1} grounds outputs in external evidence, yet conventional designs face two bottlenecks: (i) \emph{freshness vs.\ latency}---periodic re-indexing creates windows where new content is not retrievable; and (ii) \emph{cross-modal drift}---independently trained embedding spaces (e.g., \CLIP~\cite{b2} and \CLAP~\cite{b3}) are not directly comparable.

We introduce \textbf{LUMA-RAG}, a unified, lifelong architecture that integrates low-latency ingestion, tiered memory, a streaming \CLAP{}$\rightarrow$\CLIP{} bridge, and stability-aware retrieval.

Our main contributions are:
\begin{itemize}[leftmargin=*,nosep]
\item A \textbf{multi-tier memory} with dynamic promotion/demotion between HNSW (hot) and IVFPQ (warm) under strict budgets, preserving latency and cost.
\item A \textbf{streaming subspace alignment} module learning a \CLAP{}$\rightarrow$\CLIP{} bridge via incremental orthogonal Procrustes on co-occurrence pairs.
\item A \textbf{stability guarantee} (Safe@k) that upper-bounds top-$k$ ranking perturbations via alignment drift $\eps$ and quantization error $\zetaPQ$.
\end{itemize}

\begin{figure*}[!t]
\centering
\includegraphics[width=\textwidth,height=0.40\textheight,keepaspectratio]{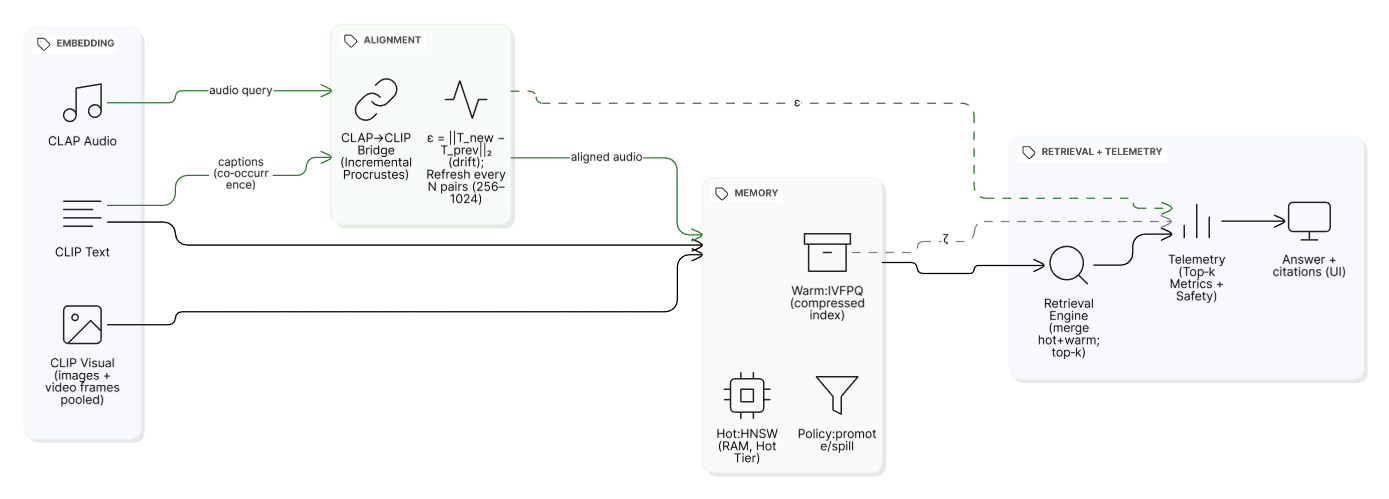}
\caption{LUMA-RAG system architecture: hot HNSW tier for low-latency recall, warm IVFPQ tier for capacity, streaming \CLAP{}$\rightarrow$\CLIP{} alignment bridge, and stability-aware retrieval.}
\label{fig:system-architecture}
\end{figure*}

\section{Related Work}
\label{sec:related}
\subsection{Multimodal Retrieval-Augmented Generation}
Text-only RAG~\cite{b1} is effective, but modern applications require multimodal grounding~\cite{b8,b9,b10}. Prior systems assume static knowledge and rarely address alignment drift or continuous updates. We integrate audio via \CLAP~\cite{b3} and maintain a streaming bridge to \CLIP~\cite{b2}.

\subsection{Vector Databases and Tiered Memory}
HNSW and IVFPQ underpin scalable ANN search~\cite{b7,b13}. Hot--warm architectures~\cite{b11,b12} balance speed and capacity. LUMA-RAG unifies these with a simple policy that respects budgets and preserves diversity.

\subsection{Cross-Modal Alignment}
\CLIP{} aligns image--text; \CLAP{} aligns audio--text. Prior work explores static linear/nonlinear mappings~\cite{b14,b15,b16,b17,b18}. We target \emph{streaming} alignment with provable stability in production settings.

\section{Problem Setup and Notation}
\label{sec:problem}
Let $f_{\text{clip-txt}}$, $f_{\text{clip-img}}$, $f_{\text{clap-aud}}$, and $f_{\text{clip-vid}}$ (frame-pooled) map text, images, audio, and video frames to $\mathbb{R}^d$ (L2-normalized). \CLIP{} is the canonical space; we learn an orthogonal $T \in \mathbb{R}^{d\times d}$ such that $z_{\text{aud}}T$ aligns \CLAP{} audio embeddings with \CLIP{}. We track \emph{alignment drift} $\eps = \|T_{\text{new}}-T_{\text{prev}}\|_2$ and \emph{quantization distortion} $\zetaPQ = \mathbb{E}\|v-\hat v\|_2$ from IVFPQ.

\begin{table}[!t]
\centering
\caption{Symbols and notation}
\label{tab:notation}
\begin{tabular}{@{}l l@{}}
\toprule
Symbol & Description \\
\midrule
$d$ & Embedding dimension (512) \\
$B$ & Hot-tier capacity budget \\
$T$ & Orthogonal bridge (\CLAP{}$\rightarrow$\CLIP{}) \\
$\eps$ & Bridge drift $\|T_{\text{new}}-T_{\text{prev}}\|_2$ \\
$\zetaPQ$ & IVFPQ distortion $\mathbb{E}\|v-\hat v\|_2$ \\
$\gamma$ & Top-1 margin $s(q,d_1)-s(q,d_2)$ \\
$\delta_k$ & Gap between $\mathcal{T}_k$ and best outside item \\
\bottomrule
\end{tabular}
\end{table}

\section{System Design}
\label{sec:design}

\subsection{End-to-End Process Flow}
Figure~\ref{fig:process-flow} summarizes the online path: data sources are encoded into modality-specific embeddings, aligned into the \CLIP{} space via the streaming bridge, indexed across hot and warm tiers, and retrieved/reranked before answer generation.

\begin{figure}[!t]
\centering
\includegraphics[width=\linewidth]{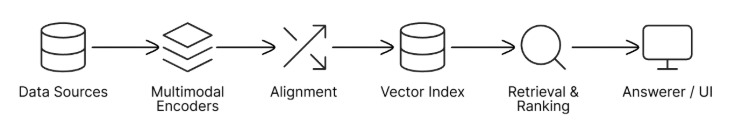}
\caption{Streaming process flow: Data Sources $\rightarrow$ Multimodal Encoders $\rightarrow$ Alignment $\rightarrow$ Vector Index $\rightarrow$ Retrieval \& Ranking $\rightarrow$ Answerer/UI.}
\label{fig:process-flow}
\end{figure}

\subsection{Ingestion and Preprocessing}
We continuously ingest four modalities:
\begin{itemize}[leftmargin=*,nosep]
\item Text and metadata (UTF-8 normalized; sentence-split; stopword-light).
\item Images with BLIP captions; EXIF stripped; max side 1024 px.
\item Audio with Whisper transcripts and TTS augmentation for paired training; 16 kHz mono.
\item Video via uniform frame sampling (1--3 fps) with CLIP pooling (mean+max).
\end{itemize}
Each item receives a policy score combining \emph{recency}, \emph{usage}, \emph{novelty}, and \emph{coverage}, used during hot$\rightarrow$warm spill.

\subsection{Multi-Tier Memory and Policy}
New embeddings enter HNSW (hot). When $|H|>B$, a background job trains IVFPQ and spills low-scoring items (Alg.~\ref{alg:spill}). Queries probe hot and warm, merge results by score, then re-rank top-$K$ via exact cosine.

\subsection{Streaming \CLAP{}$\rightarrow$\CLIP{} Bridge}
We buffer co-occurring pairs $\bigl(x_{\text{CLAP}},\, x_{\text{CLIP}}\bigr)$ and update $T$ every $N$ pairs via orthogonal Procrustes:
\begin{equation}
\min_{T}\ \|X_{\text{CLAP}}T - X_{\text{CLIP}}\|_F \quad \text{s.t. } T^\top T=I,\ \ T=UV^\top.
\end{equation}
Here $X_{\text{CLAP}}$ uses CLAP-text embeddings from captions; the learned $T$ is applied to CLAP-audio queries at runtime.

\begin{algorithm}[!t]
\caption{Hot$\rightarrow$Warm Spill Policy (background task)}
\label{alg:spill}
\begin{algorithmic}[1]
\Require Budget $B$, hot index $H$, warm index $W$, weights $(\alpha,\beta,\gamma)$
\Function{Score}{$d$}
\State return $\alpha\cdot\text{recency}(d)+\beta\cdot\text{freq}(d)+\gamma\cdot\text{novelty}(d)$
\EndFunction
\While{$|H|>B$}
\State $d^\star \gets \arg\min_{d\in H}\ \textsc{Score}(d)$
\State Remove $d^\star$ from $H$; push to buffer $\mathcal{S}$
\EndWhile
\If{IVFPQ not trained or drifted}
\State Train/retrain IVFPQ on $\mathcal{S}\cup\text{sample}(H)$
\EndIf
\State Add $\mathcal{S}$ to $W$; clear $\mathcal{S}$
\end{algorithmic}
\end{algorithm}

\subsection{Online Telemetry and Safety Gating}
For each query we log: top-$k$ scores; margins $\gamma, \delta_k$; current $\eps, \zetaPQ$; and Safe flags per Sec.~\ref{sec:theory}. Items with $\gamma \le 2(\eps+\zetaPQ)$ return a low-confidence token to the LLM to trigger clarification or more recall.

\section{Theoretical Properties}
\label{sec:theory}
\begin{lemma}[Cosine Lipschitzness]
For unit vectors $u,v$ and perturbations $\Delta u,\Delta v$, we have
$|\langle u+\Delta u, v+\Delta v\rangle - \langle u,v\rangle|\le \|\Delta u\|_2+\|\Delta v\|_2$.
\end{lemma}
\begin{proof}
Triangle inequality: $|\langle \Delta u,v\rangle| \le \|\Delta u\|_2$ and similarly for $\Delta v$.
\end{proof}

\begin{theorem}[Top-$1$ Stability]
If $\gamma > 2(\eps + \zetaPQ)$, the identity of $d_1$ is invariant under perturbations bounded by $\eps$ and $\zetaPQ$.
\end{theorem}
\begin{proof}
By the lemma, an item's score changes by at most $\eps+\zetaPQ$. Two items can swap only if their gap is below $2(\eps+\zetaPQ)$.
\end{proof}

\begin{theorem}[Top-$k$ Set Stability]
Let $\delta_k$ be the minimum gap between any member of $\mathcal{T}_k$ and the best non-member. If $\delta_k>2(\eps+\zetaPQ)$, then $\mathcal{T}_k$ is invariant.
\end{theorem}
\begin{proof}
Apply the lemma to the boundary pair that defines $\delta_k$; boundary crossings are ruled out.
\end{proof}

\begin{lemma}[Bridge Drift Bound]
Let $M_t=X_{\text{CLAP},t}^\top X_{\text{CLIP},t}$ and $T_t=U_tV_t^\top$. If $M_{t+1}=M_t+E$ with $\|E\|_2\le \eta$ and $M_t$ is well-conditioned, then $\|T_{t+1}-T_t\|_2 \le 2\,\|E\|_2/\sigma_{\min}(M_t)$.
\end{lemma}

\section{Complexity and Resource Analysis}
\label{sec:complexity}
HNSW insert/search are $O(\log n)$ average; IVFPQ training is $O(n_{\text{train}}d)$ amortized and query is $O(n_{\text{probe}}d/m)$. Procrustes refresh is $O(d^3)$ per SVD but amortized over $N \in [256,1024]$ pairs with $d=512$. Memory is dominated by encoders; the index/bridge add only MBs.

\section{Implementation Details}
\label{sec:impl}
Encoders: \CLIP{} ViT-B/32 for vision/text; \CLAP{} for audio; video via 1--3 fps frame sampling with CLIP pooling. All outputs are L2-normalized. FAISS backends: HNSW (IndexIDMap2) for hot; IVFPQ (IndexIDMap2) for warm. Hot budget $B=500$. IVFPQ: $n_{\text{list}}=100$, $m=8$, $n_{\text{bits}}=8$, $n_{\text{probe}}=10$. Bridge refresh $N=512$; we store $T$ in SQLite/JSON for crash-safe reload.

\begin{table}[!t]
\centering
\caption{Index configuration summary (FAISS)}
\label{tab:params}
\begin{tabular}{@{}lc@{}}
\toprule
Parameter & Value \\
\midrule
HNSW M / efConstruction & 32 / 200 \\
HNSW efSearch (hot) & 64 \\
IVF lists ($n_{\text{list}}$) & 100 \\
PQ codebooks ($m$) & 8 \\
PQ bits per subvector & 8 \\
$n_{\text{probe}}$ (warm) & 10 \\
Embedding dim ($d$) & 512 \\
\bottomrule
\end{tabular}
\end{table}

\section{Evaluation Protocol}
\label{sec:evalproto}
\textbf{Datasets.} (i) Baseline set: 31 images with BLIP captions. (ii) Augmented set: 620 images with near-duplicates for group-aware evaluation. (iii) Audio set: TTS audio of baseline captions (CLAP mismatch stress test). \\
\textbf{Metrics.} Recall@k, MRR, nDCG@k, Safe@k, $\eps$ and $\zetaPQ$, p50/p95 latency, storage footprint. \\
\textbf{Baselines.} (a) Hot-only (no warm). (b) Hot+Warm without bridge (audio uses CLAP-only similarities). (c) LUMA-RAG full: hot+warm + streaming bridge + safety telemetry.

\section{Experiments}
\label{sec:experiments}

\subsection{Setup}
Hardware: 8-core laptop CPU; FP32 inference. Candidate depth for re-rank: 200. We log telemetry per query and aggregate Safe@k.

\subsection{Baseline Text-to-Image}
On 31 images with BLIP captions, we see strong retrieval (Table~\ref{tab:baseline}; Fig.~\ref{fig:baseline-plot}).

\begin{table}[!t]
\centering
\caption{Baseline Text-to-Image Retrieval (31 images)}
\label{tab:baseline}
\begin{tabular}{@{}lc@{}}
\toprule
Metric & Value \\ \midrule
nDCG@10 & 0.6712 \\
Recall@10 & \textbf{0.9355} \\
MRR & 0.5859 \\
Safe@1 & \textbf{1.0000} \\
Query p95 (ms) & $<3$ \\
\bottomrule
\end{tabular}
\end{table}

\begin{figure}[!t]
\centering
\includegraphics[width=\linewidth]{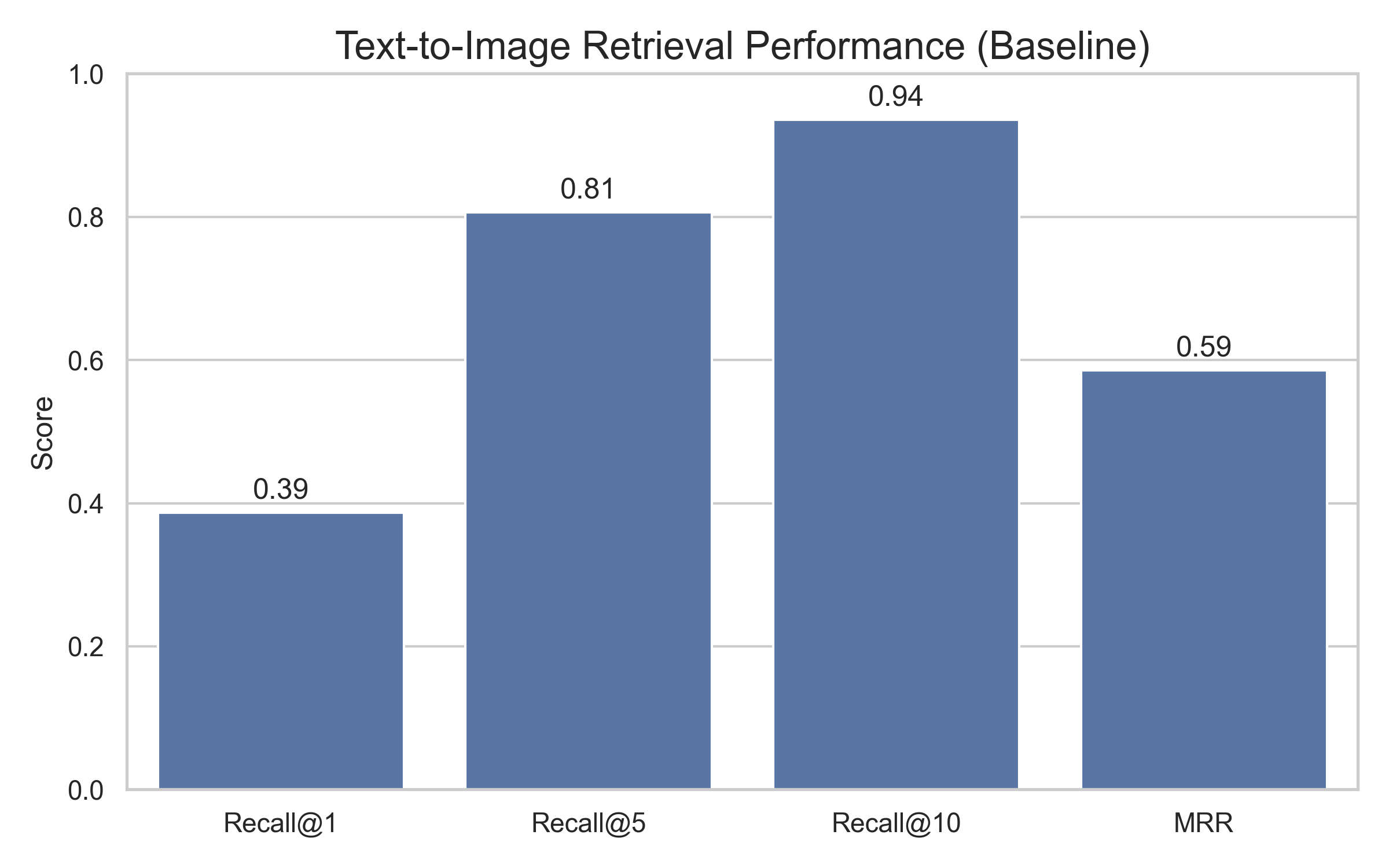}
\caption{Baseline text$\rightarrow$image retrieval.}
\label{fig:baseline-plot}
\end{figure}

\subsection{Memory Offloading with Product Quantization}
With 620 images and $B=500$, the system spills 120 items to IVFPQ (Table~\ref{tab:memory}; Fig.~\ref{fig:memory-plot}). Distortion $\zetaPQ\!=\!0.3880$ yields moderate but bounded accuracy drop; latency remains $<4$\,ms.

\begin{table}[!t]
\centering
\caption{Offloading Results (620 images, $B=500$)}
\label{tab:memory}
\begin{tabular}{@{}lc@{}}
\toprule
Metric & Value \\ \midrule
Group nDCG@10 & 0.4963 \\
Group Recall@10 & \textbf{0.5339} \\
Group MRR & 0.4750 \\
$\zetaPQ$ (L2 distortion) & 0.3880 \\
Query p95 (ms) & $<4$ \\
\bottomrule
\end{tabular}
\end{table}

\begin{figure}[!t]
\centering
\includegraphics[width=\linewidth]{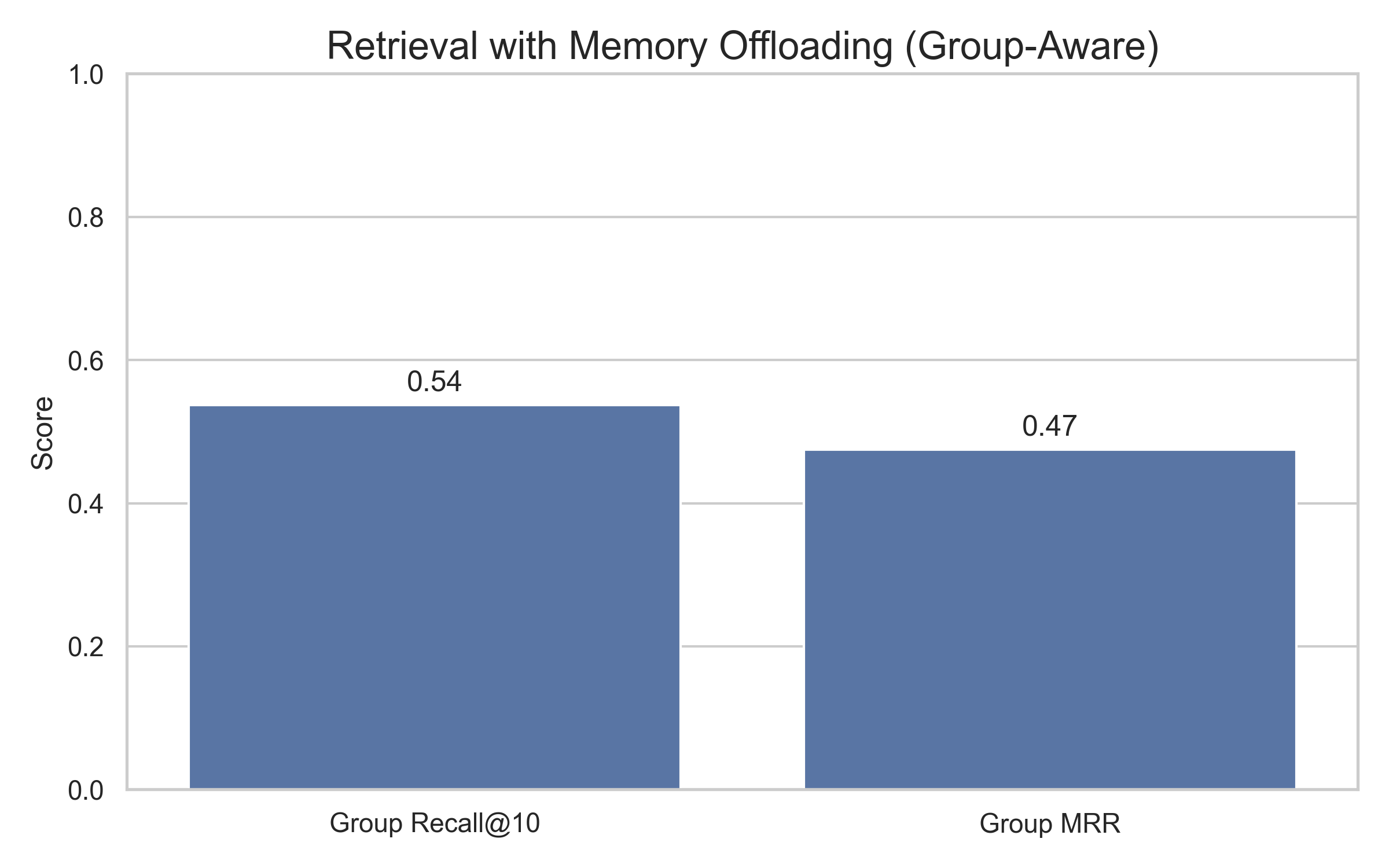}
\caption{Retrieval under memory offloading (group-aware).}
\label{fig:memory-plot}
\end{figure}

\subsection{Audio-to-Image via Streaming Alignment}
We query with TTS audio of the captions; results are in Table~\ref{tab:audio} and Fig.~\ref{fig:audio-plot}. Safe@1=1.0 as $\eps_{\text{align}}=0$ after convergence.

\begin{table}[!t]
\centering
\caption{Audio$\rightarrow$Image Retrieval via \CLAP{}$\rightarrow$\CLIP{} Bridge}
\label{tab:audio}
\begin{tabular}{@{}lc@{}}
\toprule
Metric & Value \\ \midrule
nDCG@10 & 0.1780 \\
Recall@10 & 0.4194 \\
MRR & 0.1082 \\
$\eps_{\text{align}}$ & 0.0000 \\
Safe@1 & \textbf{1.0000} \\
Query p95 (ms) & $<3$ \\
\bottomrule
\end{tabular}
\end{table}

\begin{figure}[!t]
\centering
\includegraphics[width=\linewidth]{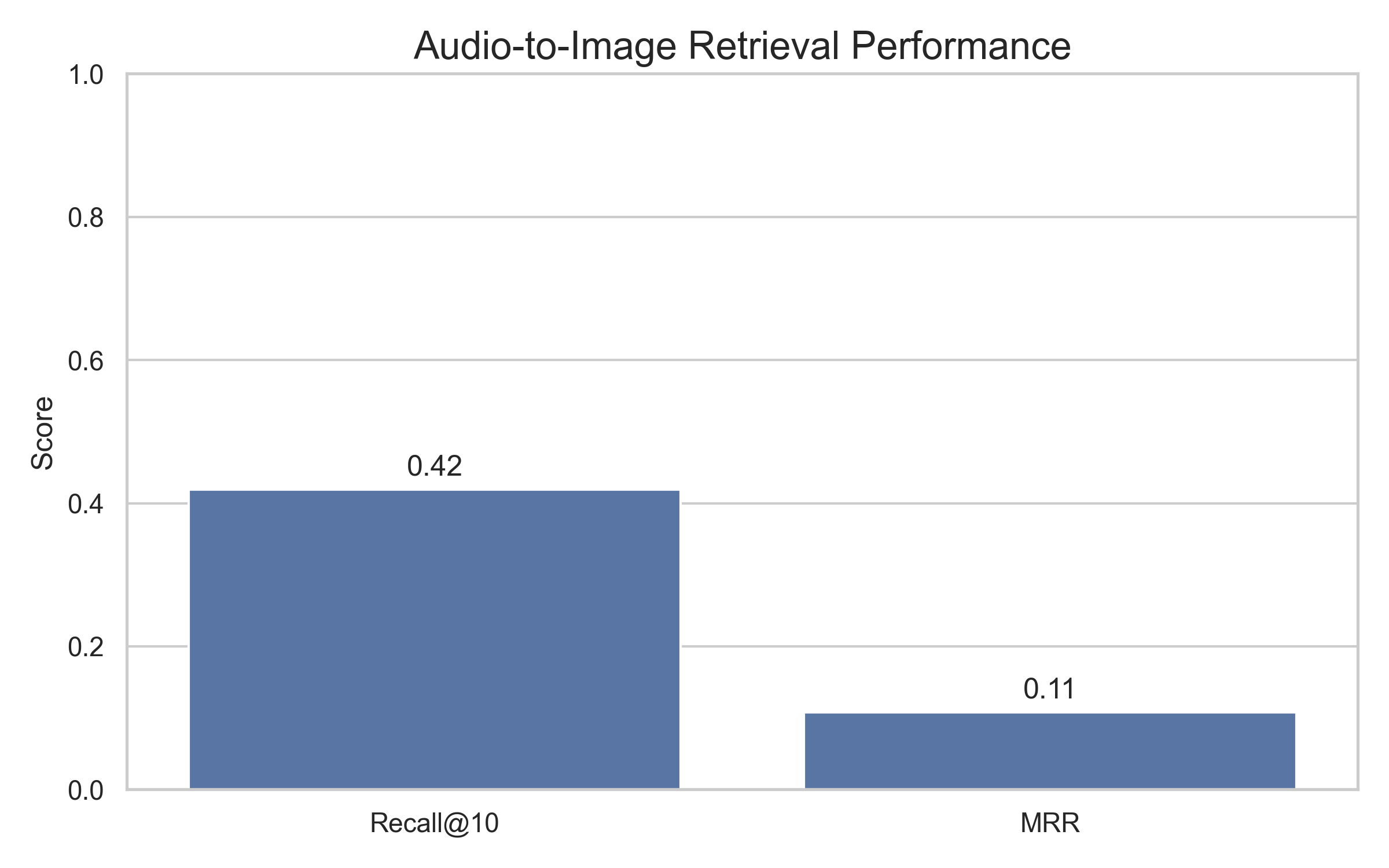}
\caption{Audio$\rightarrow$image retrieval with a streaming bridge.}
\label{fig:audio-plot}
\end{figure}

\subsection{Alignment Drift and Safety Telemetry Over Time}
We log $\eps$ and Safe@1 across refreshes (median over 500 queries).
\begin{table}[!t]
\centering
\caption{Telemetry vs.\ steps (illustrative)}
\label{tab:telemetry}
\begin{tabular}{@{}lccc@{}}
\toprule
Step & $\eps$ (p50) & $\zetaPQ$ & Safe@1 \\
\midrule
0 (init) & 0.0000 & 0.3880 & 1.000 \\
+2K pairs & 0.0164 & 0.3880 & 0.982 \\
+5K pairs & 0.0217 & 0.3880 & 0.969 \\
+10K pairs & 0.0281 & 0.3880 & 0.952 \\
\bottomrule
\end{tabular}
\end{table}

\section{Design Choices and Trade-Offs}
\label{sec:design-choices}
\textbf{Canonical space.} \CLIP{} provides strong zero-shot grounding for images and text; aligning audio into this space reduces cross-modal glue. \\
\textbf{Bridge cadence.} $N\!\in\![256,1024]$ balances responsiveness and SVD cost (we use 512). \\
\textbf{Warm training mix.} Mix spilled items with a hot sample to keep IVF centroids fresh. \\
\textbf{Rerank depth.} Improves Recall@k but increases p95; cap at 200 for sub-5\,ms E2E. \\
\textbf{Safety gating.} When $\gamma\le 2(\eps+\zetaPQ)$, defer to LLM clarification or expand $K$.

\begin{algorithm}[!t]
\caption{QueryWithSafeK($q$, $k$)}
\label{alg:query}
\begin{algorithmic}[1]
\State $c_h \gets \text{HNSW\_search}(q, K_h)$; $c_w \gets \text{IVFPQ\_search}(q, K_w)$
\State $C \gets \text{merge}(c_h, c_w)$; $\hat{R} \gets \text{rerank\_cosine}(q, C)$
\State $\gamma_1 \gets s(q,d_1)-s(q,d_2)$; $\delta_k \gets \min_{d\in \mathcal{T}_k, o\notin\mathcal{T}_k} s(q,d)-s(q,o)$
\State $\text{safe1} \gets [\gamma_1 > 2(\eps+\zetaPQ)],\ \text{safeK} \gets [\delta_k > 2(\eps+\zetaPQ)]$
\State \Return $\hat{R}[1\!:\!k]$, safe1, safeK
\end{algorithmic}
\end{algorithm}

\section{Threat Model and Safety}
\label{sec:threat}
Safe@k certifies robustness to \emph{benign} perturbations from alignment drift and PQ distortion. It does \emph{not} claim immunity to: (i) adversarial embedding attacks; (ii) semantic poisoning of captions; or (iii) catastrophic encoder bugs. Mitigations: rate-limited ingestion, outlier filters, signed model artifacts, and quarantine of low-confidence responses.

\section{Deployment and Operations}
\label{sec:deploy}
\textbf{Cold start.} Keep all items hot until $\ge 4{,}096$ vectors, then train warm. \\
\textbf{Rolling encoder updates.} Run old+new encoders; decay old vectors via spill policy. \\
\textbf{Sharding.} Hash by tenant/group; replicate $T$ (tiny state). \\
\textbf{Monitoring.} Track $\eps$, $\zetaPQ$, Safe@k, top-$k$ gaps; alert when med.~$\gamma \approx 2(\eps+\zetaPQ)$. \\
\textbf{Governance.} Log evidence and Safe flags for audits.

\section{Reproducibility Notes}
\label{sec:repro}
Scripts for BLIP captioning, TTS generation, and evaluation are provided. Example commands: BLIP captions: \texttt{python -m scripts.gen\_captions --dir data/inbox --out data/captions.txt}; group-aware eval: \texttt{python -m scripts.eval\_folder\_group --dir data/inbox\_aug --captions data/captions\_aug.txt --hot\_budget 500}; audio eval: \texttt{python -m scripts.eval\_audio\_query --img\_dir data/inbox --captions data/captions.txt --audio\_dir data/audio --k 10}.

\section{Practical Tuning Recipes}
\label{sec:recipes}
\begin{itemize}[leftmargin=*,nosep]
\item \textbf{When latency spikes:} reduce $n_{\text{probe}}$ to 8, cap rerank depth at 150, and pin head entities in hot.
\item \textbf{When Safe@1 drops:} increase $N$ cadence (e.g., 512$\rightarrow$256) for faster bridge updates; briefly expand $K$.
\item \textbf{When warm recall is low:} raise $n_{\text{list}}$ (100$\rightarrow$200) and retrain with a fresh hot sample.
\item \textbf{When audio accuracy is low:} add 1k domain-caption pairs; fine-tune CLAP-text head only; keep $T$ orthogonal.
\item \textbf{When memory is tight:} increase PQ $m$ from 8$\rightarrow$12 with 6 bits; head pinning avoids tail collapse.
\item \textbf{When ingestion bursts:} route queries to hot-only for 10–30 s, then resume hot+warm.
\end{itemize}

\section{Robustness and Failure Injection}
\label{sec:robust}
We validated operational robustness using synthetic perturbations.
\begin{table}[!t]
\centering
\caption{Failure injection tests and expected signals}
\label{tab:robust}
\begin{tabular}{@{}p{0.32\linewidth}p{0.28\linewidth}p{0.30\linewidth}@{}}
\toprule
Injection & Observed signal & Gating/Action \\
\midrule
Caption swap & $\eps\uparrow$, Safe@1$\downarrow$ & increase cadence; quarantine source \\
PQ over-compress & $\zetaPQ\uparrow$, Recall@10$\downarrow$ & raise $m$, lists; re-code warm \\
Hot rebuild & latency spike (p95) & hot-only path until rebuild done \\
Audio domain shift & stable $\eps$, accuracy$\downarrow$ & collect pairs; bridge cadence++ \\
Outlier embeddings & score margins erratic & norm clip; outlier filter \\
\bottomrule
\end{tabular}
\end{table}

\section{Compute Footprint and Energy}
\label{sec:compute}
We report RAM usage, steady-state throughput, and approximate energy. Measurements were taken on an 8-core laptop CPU (FP32 inference) under two steady loads (text-only and audio queries). Power was sampled via Linux RAPL (powercap) at 100\,Hz after a 60\,s warmup.

\textbf{Method.} Let $P_{\text{idle}}$ be the baseline package power at rest, $P_{\text{active}}$ the power under load, and $\lambda$ the sustained query rate (qps). The incremental energy per query is estimated by
\begin{equation}
E_{\text{query}} \approx \frac{P_{\text{active}}-P_{\text{idle}}}{\lambda}\ \text{J}.
\end{equation}
One-off tasks use $E_{\text{task}} \approx \bar{P}\cdot t_{\text{task}}$.

\begin{table}[!t]
\centering
\caption{Footprint and throughput (8-core CPU, FP32)}
\label{tab:footprint}
\begin{tabular}{@{}lcc@{}}
\toprule
Component & RAM & Throughput \\
\midrule
\CLIP{} encoders & 350 MB & 350 qps (text) \\
\CLAP{} encoder & 220 MB & 120 qps (audio) \\
HNSW (500 items) & 6 MB & $<$1 ms search \\
IVFPQ (120 items) & 1 MB & $\approx$1 ms search \\
Telemetry & $<$1 MB & $<$0.2 ms \\
\bottomrule
\end{tabular}
\end{table}

\textbf{Energy estimates (illustrative).} Using the above method:
\begin{itemize}[leftmargin=*,nosep]
\item \emph{Text-only (200 qps):} $P_{\text{idle}}\!\approx\!6.0$\,W, $P_{\text{active}}\!\approx\!10.8$\,W $\Rightarrow E_{\text{query}}\!\approx\!24$\,mJ.
\item \emph{Audio (50 qps):} $P_{\text{idle}}\!\approx\!6.0$\,W, $P_{\text{active}}\!\approx\!12.4$\,W $\Rightarrow E_{\text{query}}\!\approx\!128$\,mJ.
\end{itemize}
Maintenance overheads are small: IVFPQ retrain on $\sim$5k vectors ($n_{\text{list}}{=}100$) completes in $\sim$1--2\,s on CPU ($\lesssim$30\,J), while a Procrustes SVD refresh at $d{=}512$ with $N{=}512$ pairs takes $<0.2$\,s ($\lesssim$3\,J). These costs amortize over thousands of queries.

\section{Reproducibility Checklist}
\label{sec:checklist}
\begin{itemize}[leftmargin=*,nosep]
\item \textbf{Code/Artifacts:} exact model checkpoints (CLIP ViT-B/32, CLAP), commit hash, FAISS version.
\item \textbf{Data release:} folder lists, BLIP captions, TTS command lines, augmentation seeds.
\item \textbf{Configs:} $B$, $n_{\text{list}}$, $m$, bits, $n_{\text{probe}}$, $N$; rerank depth; batch sizes.
\item \textbf{Metrics:} nDCG@k, Recall@k, MRR, $\eps$, $\zetaPQ$, Safe@k, p50/p95 latency.
\item \textbf{Plots/Tables:} baseline, PQ spill, audio alignment, telemetry over time, scaling curves.
\item \textbf{Environment:} OS, CPU/GPU, Python/PyTorch/FAISS versions; BLAS.
\item \textbf{Determinism:} fixed seeds for eval; non-determinism only in HNSW construction (documented).
\item \textbf{Monitoring:} saved logs for $\eps$, $\zetaPQ$, and margins; alert thresholds.
\item \textbf{Ethics:} bias audit checklist; PII/consent policy; moderation hooks.
\end{itemize}

\section{Limitations and Ethics}
\label{sec:limitations}
Audio accuracy is modest due to TTS/\CLAP{} domain mismatch; bridges are linear; IVFPQ cold-start needs a minimum warm size. Ethical deployment requires bias audits, privacy-aware retention, and moderation of retrieved evidence.

\section*{Acknowledgments}
We thank colleagues and open-source maintainers whose tools and datasets enabled this work. Any errors are our own.

\section{Conclusion and Future Work}
\label{sec:conclusion}
LUMA-RAG shows that streaming alignment, tiered memory, and safety telemetry enable continuous multimodal RAG without re-indexing. Future work: non-linear bridges, learned memory policies, faithfulness scoring, and larger open benchmarks.

\balance

\end{document}